\documentclass{IEEEcsmag}

\usepackage[colorlinks,urlcolor=blue,linkcolor=blue,citecolor=blue]{hyperref}
\expandafter\def\expandafter\UrlBreaks\expandafter{\UrlBreaks\do\/\do\*\do\-\do\~\do\'\do\"\do\-}
\usepackage{upmath,color}
\usepackage{amsmath,amsthm,mathrsfs}
\usepackage{graphicx}
\usepackage{amsfonts}
\usepackage{cleveref}
\usepackage{tikz}
\usetikzlibrary{arrows.meta,calc}
 \usepackage{booktabs} 

\jvol{XX}
\jnum{XX}
\paper{8}
\jmonth{Month}
\jname{IEEE Computer Society Magazine}
\jtitle{Computing in Science \& Engineering}
\pubyear{2026}

\newtheorem{theorem}{Theorem}
\newtheorem{lemma}{Lemma}
\newtheorem{definition}{Definition}
\newtheorem{proposition}{Proposition}
\newtheorem{remark}{Remark}
\newtheorem{corollary}{Corollary}

\newcommand{\LRs}[1]{\left[ #1 \right]}
\newcommand{\LRp}[1]{\left( #1 \right)}
\newcommand{\LRc}[1]{\left\{ #1 \right\}}

\newcommand{\nor}[1]{\left\| #1 \right\|}
\newcommand{\snor}[1]{\left| #1 \right|}

\newcommand{\mc}[1]{\mathcal{#1}}

\usepackage[most]{tcolorbox}
\newtcolorbox{keypoints}[1][]{
 colback=blue!5!white,
 colframe=blue!75!black,
 fonttitle=\bfseries,
 title=#1,
 sharp corners,
 boxrule=0.5pt,
 left=3pt,
 right=3pt,
 top=3pt,
 bottom=3pt
}

\newtcolorbox{warningbox}[1][]{
 colback=red!5!white,
 colframe=red!75!black,
 fonttitle=\bfseries,
 title=#1,
 sharp corners,
 boxrule=0.5pt,
 left=3pt,
 right=3pt,
 top=3pt,
 bottom=3pt
}

\newtcolorbox{tipbox}[1][]{
 colback=green!5!white,
 colframe=green!75!black,
 fonttitle=\bfseries,
 title=#1,
 sharp corners,
 boxrule=0.5pt,
 left=3pt,
 right=3pt,
 top=3pt,
 bottom=3pt
}

\begin{document}

\sptitle{Controversies in AI for Science and Engineering}

\title{The AI Research Assistant: Promise, Peril, and a Proof of Concept}

\author{Tan Bui-Thanh}
\affil{ Department of Aerospace Engineering \& Engineering Mechanics, and Oden Institute for Computational Engineering \& Sciences, University of Texas at Austin, Austin, TX, 78712, USA}

\markboth{AI IN SCIENTIFIC COMPUTING}{HUMAN-AI COLLABORATION IN MATHEMATICS}

\begin{abstract}
Can artificial intelligence truly contribute to creative mathematical research, or does it merely automate routine calculations while introducing risks of error? We provide empirical evidence through a detailed case study: the discovery of novel error representations and bounds for Hermite quadrature rules via systematic human-AI collaboration.

Working with multiple AI assistants, we extended results beyond what manual work achieved, formulating and proving several theorems with AI assistance. The collaboration revealed both remarkable capabilities and critical limitations. AI excelled at algebraic manipulation, systematic proof exploration, literature synthesis, and LaTeX preparation. However, every step required rigorous human verification, mathematical intuition for problem formulation, and strategic direction.

We document the complete research workflow with unusual transparency, revealing patterns in successful human-AI mathematical collaboration and identifying failure modes researchers must anticipate. Our experience suggests that, when used with appropriate skepticism and verification protocols, AI tools can meaningfully accelerate mathematical discovery while demanding careful human oversight and deep domain expertise.

\end{abstract}

\maketitle

\chapteri{T}he integration of artificial intelligence into scientific research has triggered a fundamental debate about the nature of scholarly work. Nowhere is this controversy more acute than in mathematics, a field that prizes rigor, creativity, and human insight. Can large language models (LLMs) like \texttt{GPT-4} and \texttt{Claude} meaningfully contribute to mathematical discovery, or do they merely produce plausible-sounding but potentially flawed outputs that mislead researchers?

This paper addresses these questions not through philosophical argument but through documented experience. The closest pure-math counterpart to our work is \cite{DaviesNature2021}, though the latter does not provide the same level of detail, the patterns of success and failure, etc.  Over the past two months, we conducted a systematic exploration of human-AI collaboration in mathematical research that produced novel theoretical results for Hermite quadrature error estimation. Rather than hide the AI's role, as might be tempting given current academic norms, we provide unusual transparency about the entire research process, including both successes and failures.

While systematic evaluations on benchmarks \cite{DaviesNature2021, AlphaProof2025,HendrycksMAT2021,FreitagNeurIPS2023} assess AI capabilities in controlled settings, our detailed case study documents the messy reality of using AI for actual research with all its iterative refinements, dead ends, and verification challenges.
The broader AI math reasoning literature provides the context, while our case study provides the lived experience. Both are necessary for understanding the full picture.

\section{THE CONTROVERSY: PROMISE AND PERIL}

The use of AI in mathematics research has generated polarized responses. Optimists point to AI's potential to accelerate discovery. Recent work demonstrates diverse AI assistance approaches: machine learning systems identified patterns in knot theory that guided mathematicians to prove new theorems connecting algebraic and geometric invariants \cite{DaviesNature2021}, and formal theorem proving systems using reinforcement learning achieved silver-medal performance at the 2024 International Mathematical Olympiad \cite{AlphaProof2025}.


Skeptics raise serious concerns. LLMs are known to ``hallucinate'' plausible but incorrect mathematical statements. They lack genuine understanding and cannot reliably verify their own outputs. Critics  worry that over-reliance on AI may atrophy human mathematical intuition, introduce subtle errors that escape detection, or simply replace deep thinking with superficial pattern matching. For example,
the authors of MATH benchmark \cite{HendrycksMAT2021}, consisting of 12,500 competition mathematics problems, found that even the largest GPT-3 models achieved only 3--7\% accuracy, and they concluded that "accuracy remains relatively low, even with enormous Transformer models." Critically, models exhibited severe overconfidence: GPT-2 assigned near-100\% confidence to both correct and incorrect answers, achieving only 68.8\% AUROC (Area Under the Receiver Operating Characteristic curve) in distinguishing them, which is barely above random chance at 50\%.  

A more relevant example for this paper is the comprehensive evaluation \cite{FreitagNeurIPS2023} of ChatGPT and GPT-4 on 709 graduate-level mathematics problems. They found average performance of only 3.20/5.0 for ChatGPT and 4.15/5.0 for GPT-4, concluding that LLM mathematical capabilities are "well below the level of a graduate student." Performance degraded particularly on proof-based questions and complex calculations. The authors also stated that LLMs "almost never expressed any form of uncertainty, even if output has been completely wrong", which implies that LLMs present incorrect results with identical confidence to correct ones.

Both perspectives contain truth. The question is not whether to use AI, as these tools are already reshaping research practices, but how to use them effectively, productively, and responsibly.

\section{OUR APPROACH: STRUCTURED HUMAN-AI COLLABORATION}

We propose a framework for productive human-AI mathematical collaboration based on clear division of responsibilities:

\textbf{Human Responsibilities:}
\begin{itemize}
\item Problem formulation and research direction
\item Mathematical intuition and strategic decisions 
\item Verification of all AI-generated outputs
\item Assessment of proof validity and novelty
\item Final judgment on mathematical correctness
\end{itemize}

\textbf{AI Responsibilities (under human supervision):}
\begin{itemize}
\item Algebraic and symbolic manipulation
\item Systematic exploration of proof strategies
\item Literature search and synthesis
\item LaTeX formatting and document preparation\footnote{In particular, I asked \texttt{Claude} to convert my first draft of the paper to the CiSE magazine format and reorganize the draft content (which was originally written for a math magazine) to fit the special issue context.  I further instructed it to rewrite the abstract and introduction to reflect the title and the context, pointed out the potential audiences of the article, and told it about my role and AI role. All conceptual contents, arguments, and interpretations are my own. Every AI-drafted sentence was reviewed, and it is either revised or removed if it does not match my tone/idea/intuition. The narrative, case study selection, and critical assessments are entirely my own. This experience reinforced a key lesson: AI can accelerate the mechanical aspects of writing, but the intellectual content must remain under human authorship and verification.}
\item Generation of numerical examples
\end{itemize}

This division mirrors the relationship between a senior researcher and a highly capable but unreliable research assistant: the AI does extensive algebraic and symbolic manipulations, but the human maintains ultimate responsibility and authority.

\section{CASE STUDY: HERMITE QUADRATURE ERROR}

To demonstrate this framework in action, we present our work on Hermite interpolation-based quadrature rules. The technical problem is substantial: deriving exact error representation and improved error bounds for numerical integration when both function values and derivatives are available.

\subsection{Initial Problem Formulation (Human)}

The research began when I came across the article ``An elementary proof of error estimates for the trapezoidal rule" by Cruz-Uribe and Neugebauer \cite{cruz2003elementary} while preparing Chapter 23 of my \href{https://users.oden.utexas.edu/~tanbui/books/AdjointBook.pdf}{adjoint book} \cite{BuiThanhAdjointBook}. 
Their key innovation was using reverse integration by parts (RIBP) to derive error bounds for the trapezoidal rule requiring only boundedness of $f'$ rather than $f''$. This elementary approach, accessible to calculus students, prompted me to ponder: could multiple applications of RIBP yield similar elementary derivations for Hermite quadrature, which uses both function and derivative values at endpoints? 
I jotted down a detailed idea for how to extend this work to Hermite quadrature, but did not work out the mathematical details. 

The idea was then assigned as an undergraduate summer-intern project. Over approximately ten weeks of manual work, the students successfully derived exact error representations and explicit error bounds for the cases $n=2$ and $n=3$, requiring only boundedness of $f^{(2)}$ and $f^{(3)}$ respectively (rather than the classical requirement of $f^{(4)}$ and $f^{(6)}$) \cite{VillatoroKrishnanunniBuiThanh}. However, they encountered prohibitive algebraic complexity when attempting to extend to general $n$, noting in their draft: ``This ended up being very messy to show for a general case.'' In what follows, ``my students'' refers to these two researchers.

The central question I assigned to my students was: could we devise an elementary derivation of the exact error representation of Hermite quadrature in the spirit of \cite{cruz2003elementary}? A byproduct would be then replacing the restrictive requirement on the boundedness of the $2n$-th derivative with an $n$-th derivative of the integrand, making the results applicable to a broader class of functions.

This question required mathematical intuition and knowledge of the field. No AI system formulated this research direction. Instead, it emerged from human understanding of the problem domain and recognition of a gap in existing theory.

\subsection{Literature Review (Human and AI Collaboration)}
For this paper, we had already completed a literature review of work related to \cite{cruz2003elementary} and its extensions to Simpson's rule \cite{ELSINGERJASONR.2007AEPO,HaiD.D.2008AEPo}, so we did not ask AI to assist with the literature search---though, in hindsight, it could have been useful and would have saved a significant amount of time. In general, AI (e.g. LLMs) tools can help find papers of similar type, summarize their main contributions, and even draft a state-of-the-art overview with BibTeX citations. 
Specifically, AI can synthesize patterns across papers and identify conceptual connections that keyword search might miss. However, they sometimes hallucinate citations or misattribute results. Indeed, in a separate project, I found that some AI-suggested references did not actually exist, and others did not support the contributions attributed to them. AI can therefore be useful for generating a candidate bibliography and an initial synthesis, but every reference and every claimed contribution must be verified through traditional search engines and human review. When I do use AI for literature review, I cross-check results across multiple AI systems and then follow up with an independent search (e.g., \texttt{Google Scholar}) to catch omissions and confirm details.

\textbf{AI Success:} Comprehensive literature coverage, identification of research gap.

\textbf{AI Limitation:} Required human verification of citation accuracy and relevance assessment.

\subsection{Proof Strategy Development (Iterative Human-AI)}

The Newton--Cotes formulas, using function values at endpoints, are among the most popular numerical integration methods.
We now consider the setting in which one has both function values and derivatives at endpoints. The natural extension of Newton--Cotes is Hermite interpolation \cite{BirkhoffDeBoor1964,Schumaker1973,Lourakis2007,GasperRahman2004,Davis1975,StoerBulirsch1993} for $f\LRp{x}$. 
In particular, we consider the following quadrature rule:
\begin{equation}
 \int_a^b f(x)\, dx=\int_a^b H_n(f;x)\, dx+E_n, 
 \label{eq:hermiteQuadrature}
\end{equation}
where $H_n(f;x)$ is the two-point Hermite interpolating polynomial constructed using the following information on the function and the values of its first $n-1$ derivatives: 
$f(a),\ f'(a),\dots, f^{(n-1)}(a)$ and $f(b),\ f'(b),\dots ,f^{(n-1)}(b)$, and $E_n$ is the corresponding integration error by replacing $f(x)$ with its Hermite interpolation $H_n\LRp{f;x}$.

An expression for the error in \cref{eq:hermiteQuadrature} is available in \cite{Davis1975, atkinson}. 
Lampret et al.\ \cite{LampretVito2004AItH} considered a quadrature rule referred to as the ``composite Hermite rule,'' which applies when both function values $f(a),\ f(a+\frac{b-a}{k}),\ f(a+2\frac{b-a}{k}),\dots ,f(b)$ and the endpoint derivatives $f'(a),\ f'(b)$ are available, where $k$ is the number of subintervals in $[a,\ b]$. 
Note that the case $k=1$ in Lampret et al.\ \cite{LampretVito2004AItH} corresponds to taking $n=2$ in \cref{eq:hermiteQuadrature}, and the resulting error bounds agree with those in \cite{Davis1975, atkinson}. 
However, these error bounds require boundedness of the $2n$-th derivative of the integrand.

The core mathematical insight that I had was the following: if reverse integration by parts (RIBP) could yield elementary proofs for trapezoidal rule in \cite{cruz2003elementary}, could \emph{multiple} RIBP yield an elementary derivation of Hermite quadrature with exact error representations and milder conditions, by setting free parameters from RIBP properly?

This hypothesis was human-generated, but AI systems played a crucial role in exploring its consequences. I didn't carry out the math,
and thus did not know whether the systems of equations for free parameters are well-posed and solvable in closed-form. 
My intuition suggested that the answer was yes
when I assigned this task to my students.
In their \href{https://users.oden.utexas.edu/~tanbui/books/HermiteQuadrature.pdf}{original draft} \cite{VillatoroKrishnanunniBuiThanh}, my students had:
\begin{itemize}
\item Derived exact error formulas for $n=\LRc{2,3}$ using the coefficient matching approach,
\item Found explicit values of free parameters for cases $n=\LRc{2,3}$,
\item Attempted but failed to find closed-form solutions for free parameters for general $n$ due to algebraic complexity,\footnote{In their original remark: "This ended up being very messy to show for a general case.
So we left it out for now".}
\item Proposed numerical algorithms as an alternative, though it was not clear if the systems were solvable numerically.
\end{itemize}
 
My responsibility was then to review this draft around mid-December 2025. Since this was a busy time, I wanted to see whether AI could help me speed up the task. To my surprise and pleasure, under my direction, AI helped more than I expected. 

\begin{definition}[Two-Point Hermite Interpolating Polynomial]
\label{defi:her_def}
 Let $a$ and $b$ be distinct. The Hermite interpolating polynomial is defined by:
 \begin{equation}
 \begin{aligned}
 &H_n(f;x) = (x-a)^{n}\sum_{k=0}^{n-1}\frac{B_k(x-b)^k}{k!} + (x-b)^{n}\sum_{k=0}^{n-1}\frac{A_k(x-a)^k}{k!},\\
 &\mathrm{with}\quad A_k = \frac{d^k}{dx^k}\LRs{\frac{f(x)}{(x-b)^{n}}}_{x=a}, \quad B_k = \frac{d^k}{dx^k}\LRs{\frac{f(x)}{(x-a)^{n}}}_{x=b}.
 \end{aligned}
 \label{her_def}
 \end{equation}
\end{definition}
An immediate consequence is
\begin{align*}
 H_n(f;a) = f(a),\ H_n'(f;a) = f'(a), \dots, H_n^{(n-1)}(f;a) = f^{(n-1)}(a), \\
 H_n(f;b) = f(b),\ H_n'(f;b) = f'(b), \dots, H_n^{(n-1)}(f;b) = f^{(n-1)}(b).
\end{align*}
For testing purposes, I gave two AI platforms\footnote{Throughout this paper, I used two AI systems: Claude (Anthropic) and ChatGPT (OpenAI). I refer to them generically as ``AI platforms" as the specific system is not relevant to the discussion.} \Cref{defi:her_def} and asked them to verify the preceding derivative matching result. Interestingly, one of them immediately gave the correct proof and the other didn't. I pointed out that the proof was not correct, and the latter then provided a slightly different, but correct, proof.

A few remarks are in order. First, the AI models' proofs, though slightly different, used the same notation. This may imply that they were trained on similar data. Second, it is important that we verify and assess the validity of an AI proof (or answer), as they do make mistakes (all the time). Finally, to my pleasure, one of the proofs is slightly different from mine, that is, AI could provide/explore different proof ideas that are much faster than human.

Next, from the well-known Hermite interpolation error \cite{Davis1975, atkinson}:
\begin{equation}
 f(x) - H_n(f;x) = \frac{f^{(2n)}(\varepsilon)}{(2n)!}\LRp{(x-a)(x-b)}^n,
 \label{her_error}
\end{equation}
where $\varepsilon \in [a,b]$, we have an exact error representation for Hermite quadrature:
\begin{multline}
\int_a^b f(x)\, dx=\int_a^b H_n(f;x)\, dx \\+\int_a^b\frac{f^{(2n)}(\varepsilon)}{(2n)!}\LRp{(x-a)(x-b)}^n\, dx, 
\label{eq:int_h}
\end{multline}
which, however, requires $f(x)$ to have a $2n$-th derivative (almost everywhere) on $[a,\ b]$. 
Consequently, for certain functions such as $\snor{x}^{3/2}$ on the interval $[-1,\ 1]$, the representation for $E_n$ in \cref{eq:int_h} is not valid. Another example is 
\[ 
f(x) := \int_0^x \log\LRp{t-1/2}\,dt
\]
on the interval $[0,\ 1]$, for which the classical error \cref{her_error} does not exist for $n=1$, though the Hermite interpolation (simple linear interpolation for this case), and thus its error, is well-defined.

\textbf{The reverse integration by parts approach:} The classical error representation \cref{eq:int_h} requires $f^{(2n)}$. Our goal is to derive error representations requiring only $f^{(n)}$. The key insight comes from  \cite{cruz2003elementary} for the trapezoidal rule: instead of using polynomial approximation theory, they applied integration by parts ``backwards'' (starting from the integral and working toward boundary terms) to express the error in terms of $f'$ rather than $f''$. We extend this idea by applying RIBP multiple times---$n$ iterations yield boundary terms involving $f, f', \ldots, f^{(n-1)}$ at both endpoints, plus a remaining integral involving only $f^{(n)}$. The challenge is determining free parameters in the RIBP formula so the boundary terms exactly match the Hermite quadrature weights.

\subsection{Hermite Polynomial Integration (AI with Human Verification)}

We are going to derive an exact representation for $E_n$ that requires only the $n$-th derivative. To that end, let us first write $\int_a^b H_n(f;x)\, dx$ in the form
\begin{equation}
\int_a^b H_n(f;x)\, dx=\sum_{j=0}^{n-1}w_j^a f^{(j)}(a)+\sum_{j=0}^{n-1}w_j^b f^{(j)}(b), 
\label{eq:her_rewrite}
\end{equation}
where $f^{(j)}(a)$ denotes the $j$-th derivative of $f(x)$ evaluated at $x=a$, and $w_j^a,\ w_j^b$ are the weights to be determined.
\begin{proposition}
\label{propo:first_p}
 Consider the Hermite interpolating polynomial in \cref{her_def} for a given $n$. 
 Then the weights $w_j^a$ and $w_j^b$ in \cref{eq:her_rewrite} are given by:
\[ w_j^a=(b-a)^{j+1}n\sum_{k=j}^{n-1}\binom{k}{j}\frac{(n+k-j-1)!}{(n+k+1)!}, \quad
 w_j^b=(-1)^j w_j^a,
\]
where $\binom{k}{j}=\frac{k!}{j!(k-j)!}$. Consequently,
\[
\int_a^b H_n(f;x)\, dx=\sum_{j=0}^{n-1}w_j^a \LRs{f^{(j)}(a)+ (-1)^j f^{(j)}(b)}.
\]
\end{proposition}
Since my students had already proved \cref{propo:first_p} using direct integration of the Hermite polynomial with Beta function substitutions and general Leibniz rule, I asked one of the AI platforms to verify their proof rather than derive it independently. Its first response was "I have carefully checked your proof, and it looks correct!" with a few remarks that all steps are correct. It then said: "Overall: The proof is rigorous and correct! Nice work." Such a response made me suspicious. I then asked it to follow the derivation and check the algebraic manipulation of each step to make sure that each step was actually correct.
The AI provided me with a detailed check, which sped up my review greatly. Out of curiosity, I asked these AI platforms to provide a proof of \Cref{propo:first_p} for me. They provided different, but overly complex and long proofs. It is thus important for us to know that AI may provide unnecessarily complex (and possibly wrong) proofs.

The next result I needed to verify was the general $n$-fold RIBP, which should yield boundary terms involving derivatives up to order $n-1$ and a remaining integral involving $f^{(n)}$. My students had conjectured a specific form for this general formula, but I needed to verify it.
I again asked one of the AI platforms to check it for me. It provided a long induction and concluded with reasons that \cite[eq. (20)]{VillatoroKrishnanunniBuiThanh} was not correct. In particular, it proactively tried to verify the formula for $n=3$ and found a problem, which was great. I then instructed it to find the corrected formula by integrating by parts $n$ times the following expression (which is the last term on the right hand side of \cite[eq. (20)]{VillatoroKrishnanunniBuiThanh} that I know for sure must appear after $n$-fold RIBP). The general expressions for the boundary terms are the ones I am looking for help from AI to bypass tedious algebraic calculations.
\begin{equation}
 \int_a^b(-1)^{n}f^{(n)}(x)\LRp{\frac{(x+c)^{(n)}}{n!} + \sum_{i=0}^{n-2}\delta_i\frac{x^i}{i!}}dx.
\label{eq:En}
\end{equation}
After a couple of minutes of calculating, the AI platform said: "Excellent! I've completed a thorough verification. Here's what I found: The fully correct formula is...". However, the result was not correct: it missed factorials in the boundary terms! As can be seen, AI could indeed do the algebraic computation for me, which internally could be correct, but the final result provided to me was not entirely correct. Unlike a human, it does not seem to have the capability to double-check the result! I asked another AI platform the same question, and I was pleased that it provided the correct formula:
\begin{multline}
\int_a^b f(x)\, dx 
= \\ \sum_{j=0}^{n-1} 
f^{(j)}(b)\underbrace{(-1)^j 
\left[
 \frac{(b+c)^{j+1}}{(j+1)!}
 + \sum_{i=0}^{j-1} \delta_{\,i+n-1-j}\frac{b^{i}}{i!}
\right]}_{=: \alpha_j^b} \\
\quad + \sum_{j=0}^{n-1} 
f^{(j)}(a)\underbrace{(-1)^{j+1}
\left[
 \frac{(a+c)^{j+1}}{(j+1)!}
 + \sum_{i=0}^{j-1} \delta_{\,i+n-1-j}\frac{a^{i}}{i!}
\right]}_{=: \alpha_j^a} \\
\quad + \int_a^b (-1)^n f^{(n)}(x)
\left(
 \frac{(x+c)^n}{n!}
 + \sum_{i=0}^{n-2} \delta_i \frac{x^i}{i!}
\right)\,dx,
\label{eq:iterate}
\end{multline}
where $\boldsymbol{\theta} = \LRc{c,\ \delta_0,\dots, \delta_{n-2}\in \mathbb{R}}$ are arbitrary.

\textbf{AI Success:} Identified the appropriate mathematical tools and carried out the lengthy algebraic derivations and verifications.

\textbf{Human Verification:} Provided guidance for productive verifications, checked each algebraic step, and found and corrected errors involving missing the factorials.

\subsection{Coefficient Matching (Human Insight, AI Execution)}

Now, if we can identify constants $\boldsymbol{\theta} := \LRc{c,\ \delta_0,\dots, \delta_{n-2}}$ in \cref{eq:iterate} such that the boundary terms on the right hand side match those on the right hand side of \cref{eq:her_rewrite}:
\begin{multline}
\sum_{j=0}^{n-1}w_j^a f^{(j)}(a)+\sum_{j=0}^{n-1}w_j^b f^{(j)}(b) \\ = \sum_{j=0}^{n-1}\alpha_j^a f^{(j)}(a)+\sum_{j=0}^{n-1}\alpha_j^b f^{(j)}(b), \quad \forall f \in C^n\LRs{a,b},
\label{eq:freeParams}
\end{multline} 
then together with \cref{eq:hermiteQuadrature}, the Hermite quadrature error $E_n$ is exactly \cref{eq:En}. It is important to point out that in this case $E_n$ requires the $n$th-order, instead of the $2n$th-order, derivative of $f$ to exist (almost everywhere). 
{\bf The key questions are: 1) does a solution $\boldsymbol{\theta}$  for \cref{eq:freeParams} exist?, and 2) can we find it in a closed-form expression?}

My original aim was more ambitious. In particular, we wanted to find closed form expressions for the constants from \cref{eq:freeParams}.
To that end, 
I turned to AI for help which should be much more efficient and successful than me in terms of complex algebraic manipulations/calculations. 
I started with the concrete cases $n=\LRc{2,3}$, for which my students had already determined the free parameters ($c=-(a+b)/2$, $\delta_0=-(b-a)^2/24$ for $n=2$, and corresponding values for $n=3$), and asked one of the AI platforms,
which already saw my students' solution, to find the corresponding coefficients. It followed a similar strategy by setting $\alpha_j^a = w_j^a$ and 
 $\alpha_j^b = w_j^b$ for all $j=0, \hdots, n-1$. The interesting part was that it recognized the redundancy in the equations (see rigorous discussions below), and thus reduced the number of equations to be solved. It then provided closed form solutions for $n=2$ and $n=3$ in a blink of an eye! I verified the results and they were correct. Encouraged by this success, I then asked it to find closed form solutions for $n=\LRc{4,5}$. Again, it provided the solutions in a blink of a eye, which were again correct after my verification. Before we proceed to general $n$, let us first discuss the redundancy in \cref{eq:freeParams}, which will be useful for AI later.
 
 The redundancy can be predicted from two places: 1) the role of $a$ and $b$ in \Cref{defi:her_def} can be interchanged, and 2) the equality (up to a sign) of $w_j^a$ and $w_j^b$ in \Cref{propo:first_p}. To make this rigorous, we first show that coefficient matching conditions are both necessary and sufficient for \cref{eq:freeParams} to hold.

\begin{lemma}[Matching conditions]
 \Cref{eq:freeParams} holds if and only if
 \begin{equation}
 \alpha_j^a = w_j^a, \text{ and }
 \alpha_j^b = w_j^b, \quad \forall j=0, \hdots, n-1.
\label{eq:HermiteMatching}
 \end{equation}
 \label{lem:matchingConditions}
\end{lemma}

The sufficiency is clear, but the converse turned out to be more technical. For $n=1$, we simply take a smooth cut-off function\footnote{Any of us who had prior training in distributional theory or $L^p$ spaces or Sobolev spaces knows this fact \cite{folland1999real, hirsch1976differential}.} $\phi$ such that 
$\phi = 1$ on a neighborhood of $a$,
and $\phi\equiv 0$ on a neighborhood of $b$, and then take $f(x) = \phi(x)$ to conclude from \cref{eq:freeParams} that $\alpha_0^a = w_0^a$. Next, simply reflecting $\phi$ across the point $x = \frac{a+b}{2}$, i.e. taking $f(x) = \phi\LRp{a + b -x}$ yields $\alpha_0^b = w_0^b$. For general $n$, the idea is similar. In particular, we need to test \cref{eq:freeParams} with different functions $f$ such that each time we either have $\alpha_j^a = w_j^a$ or $\alpha_j^b = w_j^b$, for all $j=0, \hdots, n-1$. I spent a good amount of time thinking but was not successful in constructing these functions. I gave one of the AI platforms a try, and to my pleasant surprise, it could construct these functions from $\phi$! The construction (see \cref{eq:fj}) is in fact simple (now I know) and I wonder if the AI platform has seen this as part of its training data. 

\begin{proof}[Proof of \Cref{lem:matchingConditions}]
 The sufficiency is clear, and we focus on the necessity. We start by defining the difference linear functional
\[
D(f):=\sum_{j=0}^{n-1}(\alpha_j^b-w_j^b) f^{(j)}(b)
 + \sum_{j=0}^{n-1}(\alpha_j^a-w_j^a) f^{(j)}(a).
\]
By \eqref{eq:freeParams}, we have $D(f)=0$ for all $f\in C^n[a,b]$.
Let us define $d_j^a:=\alpha_j^a-w_j^a$, $d_j^b:=\alpha_j^b-w_j^b$, and 
\begin{equation}
f_{j_0}(x):=\frac{(x-a)^{j_0}}{j_0!}\,\phi(x).
\label{eq:fj}
\end{equation}
By construction, all derivatives of $\phi$ vanish at $a$ and $b$, and thus
\[
f_{j_0}^{(j)}(a)=\delta_{j,j_0}, \text{ and } f_{j_0}^{(j)}(b)=0, \qquad j=0,\dots,n-1.
\]
Substituting into $D(f)=0$ gives
\[
0=D(f_{j_0})
=\sum_{j=0}^{n-1} d_j^a f_{j_0}^{(j)}(a)
+\sum_{j=0}^{n-1} d_j^b f_{j_0}^{(j)}(b)
=d_{j_0}^a,
\]
so $d_{j_0}^a=0$. Since $j_0$ was arbitrary, $d_j^a=0$ for all $j$.
Similarly, i.e. reflecting $f_{j_0}$ across the point $x = \frac{a+b}{2}$, we can show that $d_j^b=0$ for all $j$.
Therefore $\alpha_j^a=w_j^a$ and $\alpha_j^b=w_j^b$ for all $j=0,\dots,n-1$.
\end{proof}

\subsection{Iterative development of the redundancy proof (iterative Human-AI)}

From \Cref{lem:matchingConditions} we see that there are $2n$ equations in \cref{eq:freeParams}, but we have only $n$ unknown parameters $c,\ \delta_0,\dots, \delta_{n-2}$. Thus, $n$ equations must be redundant, just like the cases $n=\LRc{2,3,4,5}$ that we have discussed. To give AI productive instruction for general $n$, I first investigated $j = 0$. The two equations $\alpha_0^a = w_0^a$ and $\alpha_0^b = w_0^b$ are:
\[
\begin{aligned}
 -(a+c) &= (b-a)n\sum_{k=0}^{n-1}\binom{k}{0}\frac{(n+k-1)!}{(n+k+1)!} = \frac{b-a}{2}, \\
 b + c &= (b-a)n\sum_{k=0}^{n-1}\binom{k}{0}\frac{(n+k-1)!}{(n+k+1)!} = \frac{b-a}{2}.
\end{aligned}
\]
Solving either of the preceding equations for $c$ gives
\[
c = -\frac{a+b}{2},
\]
independent of $n$. That is, either $\alpha_0^a = w_0^a$ or $\alpha_0^b = w_0^b$ is redundant when the other is solved. I guessed that this would hold true for any $j$, that is, either $\alpha_j^a = w_j^a$ or $\alpha_j^b = w_j^b$ is automatically valid, once the other holds. 

Assume there exists $\boldsymbol{\theta}$ such that $\alpha_j^a = w_j^a, 0 \le j \le n-1$ (to be proved below). We need somehow to tie the RIBP \cref{eq:iterate} and the Hermite quadrature \cref{eq:her_rewrite} in order to relate $\alpha_j^b $ and $ w_j^b$. I told the same AI platform that the Hermite weights $w_j^a$ and $w_j^b$ are symmetric (up to the sign), and thus the RIBP weights $\alpha_j^a$ and $\alpha_j^b$ in \cref{eq:iterate} may possess some similar symmetry, and thus the redundancy in the equations would be obvious. The AI came back to me with a proof of symmetry for the RIBP weights by reflecting $f$ around $x = \LRp{a+b}/2$. However, the proof was wrong as it assumed the symmetry of the RIBP weights. When I "confronted" it, the AI provided another proof assuming the symmetry of the polynomial kernel $P\LRp{x,\boldsymbol{\theta}}$ (to be defined below), but this was not correct either in general. I then told the AI that the redundancy, and hence the symmetry of RIBP weights, happens obviously when they match Hermite weights, and in that case, the polynomial kernel can be symmetric. However, (I told that AI that) the redundancy was what we needed to prove in the first place. I told the AI that we need to combine the RIBP \cref{eq:iterate} and the Hermite quadrature \cref{eq:her_rewrite} in order to relate $\alpha_j^b $ and $ w_j^b$.
The AI then amazed me with a clean proof that avoids all the aforementioned issues, and that I now present in a constructive manner.

The classical result \cref{eq:int_h} gives us the desired connection by taking $f \in \mc{P}^{2n-1}\LRs{a,b}$, where $\mc{P}^{2n-1}\LRs{a,b}$ is the set of polynomials of order at most $2n-1$, as the Hermite interpolation error $\frac{f^{(2n)}(\varepsilon)}{(2n)!}\LRp{(x-a)(x-b)}^n$ vanishes. In particular, we have
\begin{multline*}
\sum_{j=0}^{n-1}w_j^a f^{(j)}(a)+\sum_{j=0}^{n-1}w_j^b f^{(j)}(b) = \int_a^b H_n(f;x) dx \\= \int_a^b f(x) dx = \sum_{j=0}^{n-1}\alpha_j^a f^{(j)}(a)\\+\sum_{j=0}^{n-1}\alpha_j^b f^{(j)}(b) + \int_a^b (-1)^n f^{(n)}(x)
\underbrace{\left(
 \frac{(x+c)^n}{n!}
 + \sum_{i=0}^{n-2} \delta_i \frac{x^i}{i!}
\right)}_{=: P\LRp{x,\boldsymbol{\theta}}}\,dx,
\end{multline*}
for any $f \in \mc{P}^{2n-1}\LRs{a,b}$.
After using the assumption $\alpha_j^a = w_j^a$, we obtain
\begin{equation}
\sum_{j=0}^{n-1}\LRp{w_j^b - \alpha_j^b} f^{(j)}(b) = (-1)^n \int_a^b f^{(n)}(x) P\LRp{x,\boldsymbol{\theta}}\,dx,
\label{eq:orthogonality}
\end{equation}
for any $f \in \mc{P}^{2n-1}\LRs{a,b}$. What remains is to choose $f \in \mc{P}^{2n-1}\LRs{a,b}$ judiciously so that the right hand side vanishes while isolating $\LRp{w_j^b - \alpha_j^b}$ out on the left hand side. To annihilate the right hand side, the simplest choice is $f \in \mc{P}^{n-1}\LRs{a,b}$, to single out $\LRp{w_j^b - \alpha_j^b}$, the choice 
would not be obvious to me, hadn't I seen the polynomial factor in \cref{eq:fj}. Indeed, by taking $f = \frac{(x-a)^{j_0}}{j_0!}$ for $0 \le j_0 \le n-1$, we have
\begin{equation}
 w_{j_0}^b - \alpha_{j_0}^b = 0, \quad \forall 0 \le j_0 \le n-1,
\label{eq:wEqualAlphaAtB}
\end{equation}
which concludes the proof of the following theorem.
\begin{theorem}[Redundancy]
 $\alpha_j^a = w_j^a, 0\le j \le n-1$ if and only if $\alpha_j^b = w_j^b, 0\le j \le n-1$.
 \label{thm:redundancy}
\end{theorem}

\begin{remark}
To entertain myself, I asked if the AI could provide another proof of the exactness of Hermite interpolation for $f \in \mc{P}^{2n-1}\LRs{a,b}$ without using the classical error representation \cref{her_error}. Without letting me down, the AI gave me an alternative correct (not shown here) using more elementary facts about polynomial and induction. 
\end{remark}

Now, revisiting \cref{eq:orthogonality} and using \Cref{thm:redundancy} yield interesting properties of the polynomial kernel $P\LRp{x,\boldsymbol{\theta}}$.
\begin{corollary}
 Suppose there exists a set of parameters $\boldsymbol{\theta}$ such that $\alpha_j^a = w_j^a, 0\le j \le n-1$. Then
 \begin{equation}
 \int_a^b f(x) P\LRp{x,\boldsymbol{\theta}}\,dx = 0, \quad \forall f \in \mc{P}^{n-1}\LRs{a,b}.
\label{eq:KernelOrthogonality}
 \end{equation}
\label{coro:orthogonality}
\end{corollary}

While \Cref{coro:orthogonality} was not what I aimed for, as I wanted to prove only \Cref{thm:redundancy}, it was a nice byproduct collaboration with AI. These discussions and the success of the above proof of \Cref{thm:redundancy} also led me to an interesting symmetry of the polynomial kernel $P\LRp{x,\boldsymbol{\theta}}$.
\begin{lemma}
If the kernel orthogonality \cref{eq:KernelOrthogonality} holds, then
 $P\LRp{x,\boldsymbol{\theta}}$ is, up to a sign, symmetric around $x = (a+b)/2$, i.e.,
\[
P\LRp{a+b-x,\boldsymbol{\theta}} = 
(-1)^nP\LRp{x,\boldsymbol{\theta}}.
\] 
\end{lemma}

\subsection{Answer to the key questions (Human guide, AI extraordinary symbolic skill)}
We are now in a position to answer the {\bf key questions} posed after \cref{eq:freeParams}. The first question is straightforward. Indeed, \Cref{thm:redundancy} allows us to work with one set of $n$ equations $\alpha_j^a = w_j^a, 0\le j \le n-1$. First the number of equations is the same as the number of free parameters. Second, these equations possess an obvious triangular structure that allows us to solve numerically for all free parameters in a straightforward manner. In particular, for $j=0$, the corresponding equation $\alpha_0^a = w_0^a$ always yields $c = -\LRp{a+b}/2$ as we have shown above. For $j=1$, the corresponding equation $\alpha_1^a = w_1^a$ gives us a linear monomial equation in terms of $\delta_{n-2}$. If we continue this triangular structure, for general $1\le j \le n-1$, the corresponding equation $\alpha_j^a = w_j^a$ gives us a recursive formula for $\delta_{n-1-j}$
\[
\delta_{n-1-j} = (-1)^{j+1}w_j^a -\frac{(a+c)^{j+1}}{(j+1)!}
 - \sum_{i=1}^{j-1} \delta_{\,i+n-1-j}\frac{a^{i}}{i!}.
\]

The answer to the second key question requires either a human with an extraordinary symbolic manipulation/computational skill or a software with that skill. Since I am not the former, and neither are my students, I resorted to the latter option. I started with a greedy hands-free question for an AI platform: "Could you find simple closed forms for all free parameters in terms of $n$, $a$, and $b$?" The AI "honestly" responded: "yes, in principle, however, the resulting expressions rapidly become complicated as $n$ increases (nested finite sums with binomial coefficients), and do not appear to admit a simple closed form in $n$". I then provided a more constructive and incremental request: "Consider $j=1$, can you simplify and solve for $\delta_{n-2}$ as it does not involve other $\delta$?" After a long, complex, and tedious symbolic computations, it yielded 
\[
\delta_{n-2} = -\frac{(b-a)^2}{8(2n-1)}, \quad \forall n \ge 2,
\]
and provided the (bonus) consistency check for $n \in \LRc{2,3,4,5}$ that we worked out together previously.

Encouraged by the success, I asked the AI to find a closed form expression for $\delta_{n-3}$ from $j=2$, given the fact that we already had $\delta_{n-2}$. It again went through a series of complicated symbolic computations and found $\delta_{n-3}$. However, it didn't provide a bonus consistency check for the cases $n=\LRc{3,4,5}$ that we already knew. I asked it to do the consistency check, and it admitted that the result was wrong and then argued that finding a simple closed-form expression for $\delta_{n-3}$ passing the consistency check was not tractable. I looked into its symbolic derivation and computation carefully and I suspected its computation for $w_2^a$ was incorrect. I asked the AI to redo the calculations, and this time it provided the correct formula:
\[
\delta_{n-3} = \frac{(a+b)(b-a)^2}{16(2n-1)}, \quad \forall n \ge 3.
\]
I continued instructing the AI to find the correct expression for $\delta_{n-4}$ passing the consistency check. The detailed derivations required to arrive at the final expression for $\delta_{n-4}$ were beyond my patience, my time budget, and analytical skill (if I were to do it myself):
\[
\delta_{n-4} = \frac{(b-a)^2\bigl((b-a)^2 + (6-4n)(a+b)^2\bigr)}
 {128(2n-3)(2n-1)}, \quad \forall n \ge 4.
\]

At this point, I decided it was enough to convince myself there was no point to go further, as the expressions for these parameters were progressively more complex, and they would not add more value to the paper. 

\textbf{Initial AI Output:} gave a ``lazy" answer ``I could not find closed form solutions", then under human constructive guidance, found solutions for some parameters, and provided incorrect solution for one case.

\textbf{Human Intervention:} Detected errors by working through specific cases ($n=2, 3$) by hand, identified the pattern, and provided insights for AI to correct its errors.

This iterative refinement exemplifies effective collaboration: AI handles lengthy tedious algebraic details, and human catches errors as well as guides corrections.

\subsection{Tighter error bounds (AI provides extra facts, human exploits these facts)}

The original error bounds that my students provided assumed the uniform boundedness of the $n$-derivative (see \cite[eq. (7), Theorem 1, Theorem 2]{VillatoroKrishnanunniBuiThanh}). These bounds are the first thing one could do to derive an upper bound for the Hermite quadrature error. To see what AI could do using the same assumption for $n=2$, that is, $\snor{f^{(2)}\LRp{x}} \le M < \infty$, I tasked an AI platform to provide a tight error bound. In the output LaTeX document, the AI documented all the steps including confusion and unsuccessful tries. Two important findings for me were: i) it could come up with the same bound as in \cite[Theorem 1]{VillatoroKrishnanunniBuiThanh}
\begin{equation}
 \snor{E_2} \le M\frac{(b-a)^3(\sqrt{3})}{54},
\label{eq:originalBoundn2}
\end{equation}
 and, perhaps more important, ii) it computed the integral of the polynomial kernel $K_2\LRp{x}$ as one of its efforts (that I didn't ask and didn't know) and it said: "Interesting! The kernel integrates to zero." The latter finding (which was only obvious to me later owing to \Cref{coro:orthogonality}) was a simple finding, but it was what I could exploit to improve the original bound \cref{eq:originalBoundn2}. In particular, from my prior experience in numerical analysis, it suggested that I could add a constant for free to obtain an improved error bound:
 \begin{equation}
 |E_2| \leq \nor{\Delta}_\infty \cdot \frac{(b-a)^3\sqrt{3}}{54},
\label{eq:improvedBoundn2}
 \end{equation}
 where $\Delta(x) := {f^{(2)}(x) - c}$, with $c$ being the midrange of $f^{(2)}(x)$ defined as
 \[
 c:= \frac{\inf_{\LRs{a,b}} f^{(2)}(x) 
 +\sup_{\LRs{a,b}} f^{(2)}(x) }{2}.
 \]
 \begin{proof}
 Since $\int_a^b K_2(x)\, dx = 0$, we have
\begin{align*}
\snor{E_2} &= \snor{\int_a^b f^{(2)}(x)K_2(x)\, dx} 
= \snor{\int_a^b \Delta(x)K_2(x)\, dx} \\
&\le \nor{\Delta}_\infty\int_a^b \snor{K_2(x)}\, dx \le \nor{\Delta}_\infty \cdot \frac{(b-a)^3\sqrt{3}}{54}.
\end{align*}
 \end{proof}

 Note that the error bound in \cref{eq:improvedBoundn2} is tighter than \cref{eq:originalBoundn2} owing to the following fundamental inequality for the deviation from the midrange
 \[
 \nor{\Delta}_\infty \le \nor{f^{(2)}}_\infty.
 \]
\Cref{tab:improvedBounds} shows the improvement of the bound \cref{eq:improvedBoundn2} over the original bound \cref{eq:originalBoundn2} for five functions (that the AI picked and plotted for me). Note that the case with $\sin(2\pi x)$  shows
no improvement because the optimal constant $c$ is zero.


\begin{table}[t]
\caption{Error Bound Improvements for ($n=2$) on $[0,1]$}
\label{tab:improvedBounds}
\centering
\tablefont
\begin{tabular*}{17.5pc}{@{}lccc@{}}
\toprule
\textbf{Function} & \textbf{Original} & \textbf{Improved} & \textbf{Factor} \\
\colrule
$x^3$ & $1.92 \times 10^{-1}$ & $9.62 \times 10^{-2}$ & 2.0$\times$ \\[3pt]
$e^x$ & $8.72 \times 10^{-2}$ & $2.76 \times 10^{-2}$ & 3.2$\times$ \\[3pt]
$\sin(2\pi x)$ & $1.27$ & $1.27$ & 1.0$\times$ \\[3pt]
$\log(x+1)$ & $3.21 \times 10^{-2}$ & $1.20 \times 10^{-2}$ & 2.7$\times$ \\[3pt]
$x^4-2x^3+x^2$ & $6.42 \times 10^{-2}$ & $4.81 \times 10^{-2}$ & 1.3$\times$ \\
\botrule
\multicolumn{4}{@{}p{17.5pc}@{}}{\footnotesize Improved bound exploits kernel orthogonality to subtract optimal constant $c$ from $f^{(2)}$. No improvement for $\sin(2\pi x)$ since the optimal constant is zero.}
\end{tabular*}
\end{table}


The bound can be improved even further if a higher derivative, say $f^{(n+k)}$, with $k\le n$, is bounded by replacing $c$ with the best $(k-1)$th-order polynomial approximation to $f^{(n)}$ in the infinity norm. We can also play the same trick for the $L^2$-norm.

\section{TECHNICAL ACHIEVEMENTS SUMMARY}

For readers interested in the mathematical outcomes beyond the collaboration 
process, we briefly summarize our technical contributions:
\begin{enumerate}
    \item  Exact error representation for Hermite quadrature requiring only $n$-th 
   derivatives instead of $2n$-th derivatives
   
\item Proof that coefficient matching conditions are necessary and sufficient 
   (\Cref{lem:matchingConditions} on Matching Conditions)
   
\item Redundancy theorem showing that solving $n$ equations from $2n$ is sufficient to determine a unique set of 
   coefficients
   
\item Closed-form solutions for free parameters up to $n=4$
   
\item Orthogonality properties of the polynomial kernel

\item Improvements on the error bounds
\end{enumerate}

Detailed proofs and extensions are provided in the accompanying technical paper \cite{BuiThanhVillatoroKrishnanunni2026HermiteQuadratureError}.

\section{PATTERNS OF SUCCESS AND FAILURE}

Our experience reveals clear patterns in what works and what doesn't in human-AI mathematical collaboration. It is important to emphasize that AI is not a solution, but a highly capable but unreliable research assistant.

\subsection{Where AI Excelled}

\begin{enumerate}
 \item {\bf Algebraic Manipulation}: AI systems handled complex symbolic manipulations reliably once the setup was correct and the results are verified by human. This saves hours or days of tedious calculation.
 \item {\bf Systematic Exploration}: When asked to explore variations (different values of $n$, alternative proof strategies), AI provided comprehensive coverage
faster than a human working alone.
 \item {\bf Literature Synthesis}: AI rapidly identified relevant papers and synthesized key ideas, though citation accuracy required verification.

\item {\bf LaTeX Formatting}: AI produced well-formatted mathematical documents, handling complex equation environments and theorem styling.

\item {\bf Numerical Verification}: AI generated comprehensive test cases (not shown here) and verified specific cases much faster than a human.
\end{enumerate}

\subsection{Where AI Failed}
\begin{enumerate}
 \item {\bf Problem Formulation}: AI couldn't reliably generate research questions. Most significant mathematical insights and strategic directions came from human researchers.

\item {\bf Error Detection}: AI frequently produced outputs with subtle errors (wrong indices, incorrect signs, wrong results) that appeared plausible but were mathematically wrong.

\item {\bf Proof Validity}: AI could not reliably assess whether a proof was valid or novel. It sometimes claimed to verify results that were actually incorrect.

\item {\bf Strategic Directions}: When stuck, AI would often continue down unproductive paths. Human intuition was essential for course correction.

\item {\bf Mathematical Judgment}: In general, AI could not determine whether results were interesting, surprising, or worth publishing. It tended to flatter results/findings and sometimes even said that the results were publishable.
\end{enumerate}

\subsection{Critical Failure Mode: Plausible Nonsense}

The most dangerous failure mode was AI generating mathematical statements that appeared correct but contained fundamental errors. For example, when exploring various conditions on $f^{(2)}$ and their corresponding bounds, it made a mistake 
 that seemed convincing at first glance: assuming $f^{(2)}$ has constant sign and then yield perfect effort bound by implicitly assuming that $f^{(2)}$ was a constant. Only detailed verification revealed the error. There are many other occasions where AI provided plausible nonsense. Another example was when I wanted to show that $\alpha_j^b = (-1)^j \alpha_j^a, \qquad j=0,\dots,n-1$, and 
 the AI provided the following plausibly correct statement
\begin{quote}

Let $n\ge 1$. Suppose there exist constants 
$\alpha_j^a,\alpha_j^b$ ($j=0,\dots,n-1$) and a polynomial $K_n$
such that for every $f\in C^n[a,b]$ we have \cref{eq:iterate}.
Then the endpoint coefficients satisfies
\[
\alpha_j^b = (-1)^j \alpha_j^a, \qquad j=0,\dots,n-1.
\]
\end{quote}
\noindent with a believable proof using test functions from distributional theory, etc. The statement itself was too good to be true because that couldn't happen unless \cref{eq:freeParams} holds. When I went through the proof carefully I found a fundamentally incorrect, but plausibly correct fact that AI used (not shown here for brevity).

\begin{warningbox}[Warning Signs That AI May Be Wrong]
\begin{itemize}
\item Result seems too neat or simple for the problem complexity
\item Explanation is verbose but vague on critical details
\item Similar problems give inconsistent results
\item AI refuses to show intermediate steps
\item Numerical tests contradict theoretical predictions
\item Proof relies on unverified "well-known" facts
\end{itemize}
\end{warningbox}

This highlights a critical point: \textbf{AI assistance without expert human verification is actively dangerous.}

\subsection{Connecting Observed Failures to LLM Reliability Literature}

Our AI failure experiences align with systematic evaluations of LLM mathematical capabilities, confirming  these are general limitations rather than isolated incidents. Three specific findings in \cite{FreitagNeurIPS2023} directly match our experiences:

{\em 1. Computational errors uncorrelated with difficulty:} The authors in \cite{FreitagNeurIPS2023} found ChatGPT ``executes complicated symbolic tasks with ease'' but ``fails on simple arithmetic operations,'' with errors uncorrelated with expression/problem complexity. Computational errors occurred in 36\% of algebra problems. Our experiences match this pattern: AI correctly handled complex Beta function evaluations but dropped factorials in simpler nested summations and integration by parts.

{\em 2. Unwavering confidence despite errors:} ChatGPT ``almost never expressed any form of uncertainty, even if its output has been completely wrong.'' Our observation: AI presented the incorrect RIBP formula with complete confidence, then confidently validated its own error when asked to verify. Only explicit counterexamples forced error recognition.

\textbf{3. Logical errors in proofs:} Frieder et al. documented ``e5\_5: circular logical argument'' as a systematic failure mode. Our observation: AI's symmetry proof assumed the symmetry it claimed to prove—exactly this error type.

\textbf{Implication:} Our detailed case study illustrates what Frieder et al.'s ``3.20/5.0 rating'' means in practice: AI makes multiple errors at every stage of graduate-level mathematics, requiring constant human detection and correction. Our collaboration succeeded because we maintained the rigorous verification their evaluation prescribes as essential.

\section{LESSONS AND GUIDELINES}

Based on this experience, we offer guidelines for productive human-AI collaboration in mathematics:

\subsection{Dos}
\begin{enumerate}
 \item {\bf Maintain Verification Culture/habit}: Treat all AI outputs as hypotheses to be tested, not facts to be accepted.

\item {\bf Use AI for Leverage}: Deploy AI on tasks that are tedious but verifiable (algebra, literature search, formatting). AI proof exploration could be surprisingly productive under human supervision and rigorous verification.

\item {\bf Preserve Human Judgment}: Keep research direction, problem formulation, and final assessment under human control.

\item {\bf Document Thoroughly}: Keep detailed records of AI contributions for both credit attribution and debugging.

\item {\bf Test Extensively}: Generate numerical examples, special cases, and limiting behavior checks.

\item {\bf Iterate Rapidly}: Use AI to quickly explore multiple approaches, then apply human judgment to select promising directions.
 
\end{enumerate}

\begin{keypoints}[Dos: Five Core Principles]
\begin{enumerate}
\item \textbf{Verify Everything}: Treat all AI outputs as hypotheses requiring verification
\item \textbf{Maintain Strategic Control}: Keep research direction and problem formulation human-led
\item \textbf{Use AI for Tedious Tasks}: Deploy AI for algebra, literature search, formatting, and proof explorations
\item \textbf{Test Extensively}: Generate numerical examples and special cases
\item \textbf{Document Transparently}: Keep detailed records of AI contributions
\end{enumerate}
\end{keypoints}

\begin{tipbox}[Quick Verification Checklist]
Before accepting any AI result:
\begin{itemize}
\item Test with simple special cases where you know the answer
\item Check dimensional analysis and limiting behavior
\item Verify non-trivial intermediate steps by hand
\item Compare with existing results in limiting cases
\item Run numerical experiments if analytical verification is difficult
\end{itemize}
\end{tipbox}

\subsection{Don'ts}

\begin{enumerate}
 \item {\bf Don't Trust Blindly}: Never accept AI mathematical outputs without verification.

\item {\bf Don't Outsource Understanding}: Use AI to augment your thinking, not replace it.

\item {\bf Don't Skip Steps}: Even if AI produces an answer, work through the derivation yourself. For steps that AI skipped, either carry them out yourself or ask them to provide the details and then check.

\item {\bf Don't Ignore Errors}: When you find errors in AI output, use them as learning opportunities about failure modes.

\item {\bf Don't Publish/Accept Without Verification}: Ensure every theorem, proof, and claim has been independently verified by qualified humans.
\end{enumerate}

\section{BROADER IMPLICATIONS}

My limited experience suggests several broader implications for AI in scientific research:

\textbf{1. AI as an Amplifier, Not a Replacement:} AI tools amplify human capability but cannot replace domain expertise and judgment.

\textbf{2. Verification is Essential:} The ease of generating content makes verification more critical, not less.

\textbf{3. New Skills Required:} Researchers need to develop skills in prompting and strategic directions, verification, and critical assessment of AI outputs.

\textbf{4. Transparency Matters:} Academic culture should encourage, not discourage, transparency about AI assistance.

\textbf{5. Education Must Adapt:} We must teach students both how to use AI tools and, critically, how to work without them productively and responsibly.

\section{ANSWER TO THE TITLE QUESTION}

Is human-centered AI-assisted creative scholarly work in mathematics productive collaboration or perilous shortcut?

My personal answer: \textbf{it depends entirely on how you use it.}

With proper safeguards\textemdash expert human verification, clear division of responsibilities, appropriate skepticism, and rigorous testing\textemdash AI assistance can meaningfully accelerate mathematical discovery and verification while enhancing understanding. We completed in weeks what might have taken months without AI assistance.

However, without these safeguards, AI assistance becomes a liability. The risk of subtle errors, the atrophy of mathematical thinking, and the false confidence in unverified results can lead researchers astray.

The technology is neither inherently good nor bad. The outcomes depend on human choices about implementation and oversight.

\section{SHOULD WE ENCOURAGE HUMAN-CENTERED AI-ASSISTED SCHOLARLY WORK?}

This is clearly debatable (and even controversial). Some of us may disapprove of the idea, while others will embrace it. Historically, scholarly work has been human--human collaboration, and that has been the default model.

We have long accepted citing mathematical software and computer algebra systems like \texttt{Mathematica}, \texttt{MATLAB}, 
or \texttt{Maple} when they perform symbolic manipulations. These tools execute 
algorithms with precise syntax and problem encoding programmed by humans, yet we don't credit them as coauthors. 
Why? Because they are traditionally deterministic tools executing known procedures.

LLMs are fundamentally different. 
They offer natural language
interaction and can suggest solution strategies, but could make computational errors. 
They generate novel strategies, explore 
unexpected proof directions, and sometimes produce insights that surprise 
even expert users. The distinction between "tool" and "collaborator" becomes 
murky in this case. If we credit a human research assistant under supervision, why not credit an AI assistant performing 
the same function, as long as new knowledge or findings have been developed? In other words, as far as advancing knowledge and the state-of-the-art is concerned, should it matter if a scholarly work is a human-human or human-AI collaboration?

However, I acknowledge the counterarguments:
\begin{itemize}
 \item AI lacks intent, understanding, and responsibility.
\item AI cannot defend its contributions or verify correctness.
\item Coauthorship implies accountability that AI cannot bear.
\item There are practical issues with copyright and legal responsibility
\end{itemize}

Perhaps the solution is a new category: "AI Research Assistant" in 
acknowledgments, with detailed specification of contributions. This 
preserves transparency without granting full coauthorship status.

\section{LIMITATIONS OF THIS STUDY}
This case study has several limitations that readers should consider:
\begin{enumerate}
 \item {\bf Single Domain}: Our experience is limited to mathematical proof development. 
 Results may not generalize to experimental science, data analysis, or 
 other forms of scholarly work.

\item {\bf Expert User}: As a professor with extensive mathematical training and experience, I could 
 quickly identify AI errors. Less experienced researchers might struggle 
 with verification.

\item {\bf Specific Problem Type}: The algebraic nature of our problem played to AI 
 strengths. More conceptual or geometric problems might yield different 
 results.

\item {\bf Time Snapshot}: AI capabilities are rapidly evolving. This study reflects 
 systems available in late 2025.

\item {\bf Selection Bias}: We document a successful collaboration. Failed attempts 
 at AI-assisted research may have different patterns we didn't observe.
\end{enumerate}

While broader methodological studies across multiple problems would enhance generalizability, we believe detailed transparency about one case offers practical insights appropriate for illuminating controversies.

\section{COMPARISON WITH SPECIFIED TOOLS}

\textbf{LLMs vs. Search Engines for Literature:} LLMs can synthesize patterns across papers and identify conceptual connections that keyword search might miss. However, they sometimes hallucinate citations or misattribute results. My recommendation is to use LLMs for initial literature mapping, then verify all citations manually using traditional search and human verifications.

\noindent \textbf{LLMs vs. Computer Algebra Systems:} CAS (like \texttt{Mathematica} and  \texttt{Maple}) excel at deterministic symbolic computation but require precise syntax and problem encoding. LLMs offer natural language interaction and can suggest solution strategies, but could make computational errors. The ideal workflow combines both: use LLMs for strategy exploration and problem setup, then CAS for verifications when applicable.

\noindent \textbf{General LLMs vs. \texttt{LLEMMA} and \texttt{Minerva}:} This requires an in-depth, comprehensive, and systematic comparison, which is beyond the scope of this paper. However, the general LLMs (such as \texttt{GPT-4}, \texttt{Claude}) proved surprisingly capable when properly prompted for our mathematical derivation tasks. This suggests that for many applications, accessibility and ease of use may matter more than mathematical specialization.

\section{CONCLUSION}

We have provided empirical evidence that human-AI collaboration in mathematics can produce rigorous, novel results when properly structured. Our work on Hermite quadrature demonstrates both the potential and the limitations of current AI systems.

The key insight is not that AI can do mathematics\textemdash it cannot, at least not reliably\textemdash but that it can assist mathematicians in specific, well-defined ways under appropriate supervision. The human must remain firmly in control, providing intuition, strategic directions, verification, and judgment.

As AI capabilities continue to advance, the research community must develop best practices for responsible use. We hope this detailed account of our experience contributes to that ongoing conversation by providing concrete evidence of what works, what doesn't, and what safeguards are essential.

The future of mathematics may well involve human-AI collaboration. Our results suggest this future can be productive, if we proceed thoughtfully, responsibly, and maintain appropriate skepticism.

\section*{ACKNOWLEDGMENTS}

 The author would like to thank his students, Giancarlo Villatoro and C.G. Krishnanunni for their initial work without AI assistance in \cite{VillatoroKrishnanunniBuiThanh}. The author also would like to acknowledge the contributions of Claude (Anthropic) and ChatGPT (OpenAI) as AI research assistants in this work. All final collaborative mathematical results were verified by the author. We thank the Associate Editor and all reviewers for their constructive feedback, which has significantly improved the manuscript. This work is partially funded by the National Science Foundation award NSF-OAC-2212442, and by the Department of Energy award DE-SC0024633.

 \section{Examples of prompts and outputs}

\bibliographystyle{unsrt}
\bibliography{references}

\end{document}